%% file: main.tex
\def\year{2022}\relax
\documentclass[letterpaper]{article} % DO NOT CHANGE THIS
\usepackage{aaai22}  % DO NOT CHANGE THIS
\usepackage{times}  % DO NOT CHANGE THIS
\usepackage{helvet}  % DO NOT CHANGE THIS
\usepackage{courier}  % DO NOT CHANGE THIS
\usepackage[hyphens]{url}  % DO NOT CHANGE THIS
\usepackage{graphicx} % DO NOT CHANGE THIS
\usepackage{natbib}  % DO NOT CHANGE THIS AND DO NOT ADD ANY OPTIONS TO IT
\usepackage{caption} % DO NOT CHANGE THIS AND DO NOT ADD ANY OPTIONS TO IT
\DeclareCaptionStyle{ruled}{labelfont=normalfont,labelsep=colon,strut=off} % DO NOT CHANGE THIS
\usepackage{algorithm}
\usepackage{algorithmic}
\usepackage{amsfonts}
\usepackage{amssymb}
\usepackage{amsmath}
\usepackage{todonotes}
\usepackage{xspace}
\usepackage{booktabs}
\usepackage{subcaption}
\usepackage{bbm}
\usepackage{newfloat}
\usepackage{listings}
\floatstyle{ruled}
\newfloat{listing}{tb}{lst}{}
\floatname{listing}{Listing}
\input{def.tex}

\title{Humanly Certifying Superhuman Classifiers}

\title{Humanly Certifying Superhuman Classifiers}
\author {
    Qiongkai Xu,\textsuperscript{\rm 1,2}
    Christian Walder, \textsuperscript{\rm 1,2}
    Chenchen Xu \textsuperscript{\rm 1,2}
}
\affiliations {
    \textsuperscript{\rm 1} The Australian National University, Canberra, ACT, Australia\\
    \textsuperscript{\rm 2} Data61 CSIRO, Canberra, ACT, Australia\\
    \texttt{first.last@anu.edu.au}
}

\usepackage{bibentry}
\begin{document}

\maketitle

\input{sec0-abstract}

\input{sec1-intro}

\input{sec4-related}

\input{sec2-theorem}

\input{sec3-exp}

\input{sec5-conclusion}

\input{sec6-impact}

\bibliography{aaai22}

% \section{Acknowledgments}

\newpage
\input{appendix}

\end{document}

%% file: def.tex
\newcommand{\xqk}[1]{\textcolor{blue}{\small\textbf{[QK:]} #1}}
\newcommand{\cw}[1]{\textcolor{red}{\small\textbf{[CW:]} #1}}

\newcommand{\sref}[2]{\hyperref[#2]{#1~\ref{#2}}} 

\newcommand{\taskcolor}{\textbf{Color Classification}\xspace}
\newcommand{\taskshape}{\textbf{Shape Classification}\xspace}

\newcommand{\ellmodel}{\ell_{\mathcal M}}
\newcommand{\elloracle}{\ell_{\star}}
\newcommand{\ellaverage}{\ell_{\mathcal K}}
\newcommand{\ellincorrect}{\ell_{\times}}
\newcommand{\ellaggregate}{\ell_{\mathcal G}}

\usepackage{hyperref}  
\usepackage{cleveref}
\usepackage{ntheorem}

\newtheorem{theorem}{Theorem}
\newtheorem{lemma}[theorem]{Lemma}

\newtheorem{definition}{Definition} 
\theoremstyle{nonumberplain}
\newtheorem{proof}{Proof}

%% file: sec0-abstract.tex
%!TEX root = main.tex

\begin{abstract}
Estimating the performance of a machine learning system is a longstanding challenge in artificial intelligence research. Today, this challenge is especially relevant given the emergence of systems which appear to increasingly outperform human beings. In some cases, this ``superhuman'' performance is readily demonstrated; for example by defeating legendary human players in traditional two player games. On the other hand, it can be challenging to evaluate classification models that potentially surpass human performance. Indeed, human annotations are often treated as a ground truth, which implicitly assumes the superiority of the human over any models trained on human annotations. In reality, human annotators can make mistakes and be subjective. Evaluating the performance with respect to a genuine oracle may be more objective and reliable, even when querying the oracle is expensive or impossible.
In this paper, we first raise the challenge of evaluating the performance of both humans and models with respect to an oracle which is unobserved. We develop a theory for estimating the accuracy compared to the oracle, using only imperfect human annotations for reference. Our analysis provides a simple recipe for detecting and certifying superhuman performance in this setting, which we believe will assist in understanding the stage of current research on classification. We validate the convergence of the bounds and the assumptions of our theory on carefully designed toy experiments with known oracles. Moreover, we demonstrate the utility of our theory by meta-analyzing large-scale natural language processing tasks, for which an oracle does not exist, and show that under our assumptions a number of models from recent years are with high probability superhuman. 
%, therefore, inspiring new supervised learning paradigm and influencing the policy on human-computer collaboration.
\end{abstract}

%% file: sec1-intro.tex
%!TEX root = main.tex

\section{Introduction}
% stage of AI and challenge of evaluating supervised learning
%Taking the advantages of machine learning technologies, 
Artificial Intelligence (AI) agents have begun to outperform humans on remarkably challenging tasks; AlphaGo defeated legendary Go players~\cite{silver2016mastering,singh2017learning}, and OpenAI's Dota2 AI has defeated human world champions of the game~\cite{berner2019dota}. These AI tasks may be evaluated objectively, \textit{e.g.} using the total score achieved in a game and the victory or defeat against another player.
%\todo[inline]{We may show GPT as a case that human does to agree language model has outperformed human and likely will not agree.}
However, for supervised learning tasks such as image classification and sentiment analysis, certifying a machine learning model as superhuman is subjectively tied to human judgments rather than comparing with an oracle. This work focuses on paving a way towards evaluating models with potentially superhuman performance in classification. %, because the evaluation relies on the `ground truth' label by human annotators in place of oracle. %\xqk{We need to claim we focus on classification tasks somewhere in the paper.}
%In this paper, we focus on classification tasks and develop theory for evaluating oracle accuracy, as performance, of human and models.
% The accuracy by human is does not provide a fair and stable score for comparison.

\begin{figure}[t]
    \centering
    \includegraphics[width=\linewidth]{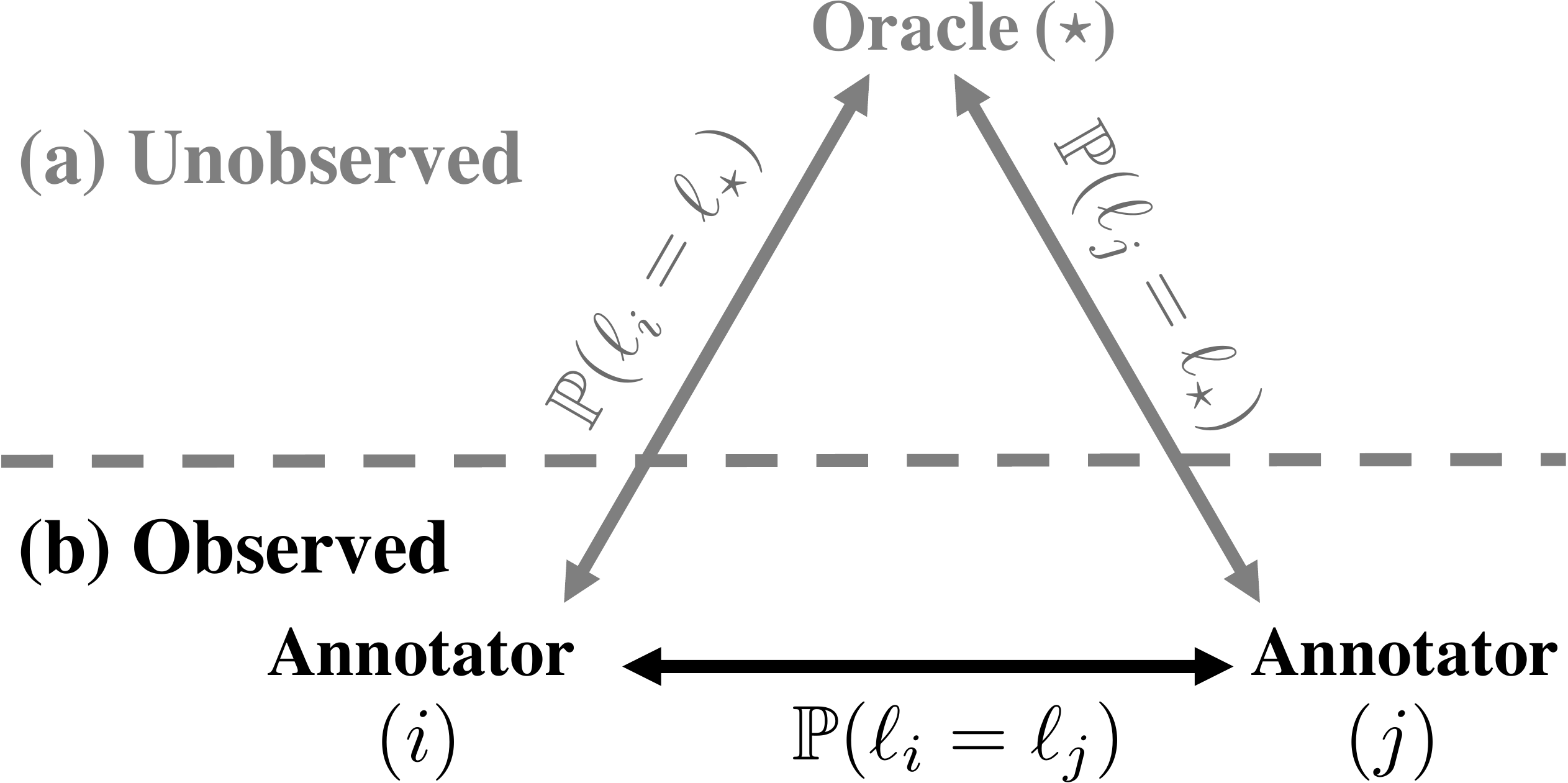}
    \caption{The relationship between \textit{a)} the oracle accuracy of the annotators, $P(\ell_i=\elloracle)$, and \textit{b)} the agreement between two annotators, $P(\ell_i=\ell_j)$. $\ell_i$ and $\ell_j$ are labels given by annotator $i$ and $j$, $\elloracle$ is the oracle label. In our setting, part \textit{a)} is unobserved ({\color[HTML]{7F7F7F}\textit{gray}}) and part \textit{b)} is observed (\textit{black}).}
    \label{fig:tease}
    % \vspace{-2.0mm}
\end{figure}

When evaluating the performance of a classification model, we generally rely on the accuracy of the predicted labels with regard to ground truth labels, which we call the \textit{oracle accuracy}.
%In order to acquire the very accurate oracle accuracy, the oracle labels should be compared.
However, oracle labels may arguably be unobservable. For tasks such as object detection, the predictions are subjective to many factors of the annotators, \textit{e.g.}, their background and physical or mental state. For other tasks, even experts may not be able to summarize an explicit rule for the prediction, such as predicting molecule toxicity and stability. %We can never expected the collected data to be perfect as oracle.
Without observing the oracle labels, human predictions or aggregated human annotations are treated as ground truth~\cite{wang2018glue,lin2014microsoft,wang2019superglue} to approximate the oracle. Such approximation mainly suffers from two disadvantages. Firstly, the quality control of human annotation is challenging~\cite{artstein2017inter,lampert2016empirical}. Secondly, current evaluation paradigms focus on evaluating the performance of models, but not the oracle accuracy of humans --- yet we cannot claim that a machine learning model is superhuman without a proper estimation on human performance.% The evaluated performance of machine learning models is less likely to surpass human performance, as human annotations are considered as `ground truth' with 100\% accuracy. We argue that such evaluation paradigm is unfair for AI models.
%However, in many real-world supervised learning tasks, the oracle label is arguably unobservable to us, \textit{e.g.}, objects in a blur image and sentiment of an ambiguous text. Therefore, the oracle accuracy of a model label is also unobservable, as illustrated in Figure~\ref{fig:tease}(a) in gray. On the other hand, we can observe the labels from annotators and calculate their agreement.

% our work (theory+case)
In this paper, we work on the setting that oracle labels are unobserved (see Figure~\ref{fig:tease}). %The oracle accuracy of annotators is unobservable, demonstrated as gray part, while the agreement between annotators, $A_i$ and $A_j$, are observable, illustrated as black part. Similarly, when evaluating machine learning models, we can only compare the model outputs with annotator labels, while its oracle accuracy labels is not observable. Note that we are interested on model or human oracle accuracy, although it is likely to be unobservable.
% primary question of classification problems: {Can we estimate the performance range of human annotators and machine learning models on \textit{oracle}}. Given multiple annotators predictions, we manage to give i) an upper bound for the averaged oracle accuracy of annotators and ii) a lower bound for the performance of a single model. Then, we implement our theorem on two synthetic vision tasks and two real-world language understanding tasks.
% We mange to grant certifications to some machine learning models that outperforms human annotators.
% theory contributions:
Within this setting, we develop a theory for estimating the oracle accuracy on classification tasks. Our theory includes \textit{i)} upper bounds for the averaged oracle accuracy of the annotators, \textit{ii)} lower bounds for the oracle accuracy of the model, and \textit{iii)} finite sample analysis for both bounds and their margin which represents the model's outperformance. We propose an algorithm to discover competitive models and to report confidence scores, which formally bound the probability that a given model outperforms the average %\cw{yes! and luckily superhuman = ``beats average human'' $\neq$ ``beats all humans'', see here: \url{https://www.merriam-webster.com/dictionary/superhuman}.}
human annotator. Empirically, we observe that some existing models for sentiment classification and natural language inference (NLI) have already achieved superhuman performance.

%% file: sec4-related.tex
%!TEX root = main.tex

\section{Related Work}
Classification accuracy is a widely used measure of model performance~\cite{han2011data}, although there are other options such as precision, recall, F1-score~\cite{chowdhury2010introduction,sasaki2007truth}, Matthews correlation coefficient~\cite{matthews1975comparison,chicco2020advantages}, \textit{etc.}. Accuracy measures the disagreement between the model outputs and some reference labels. A common practice is to collect human labels to treat as the reference. However, we argue that the ideal reference is rather the (unobserved) oracle, as human predictions are imperfect. We focus on measuring the \textit{oracle accuracy} for both human annotators and machine learning models, and for comparing the two.

A widely accepted approach is to crowd source \cite{kittur2008crowdsourcing,mason2012conducting} a dataset for testing purposes. The researchers collect a large corpus with each examples labeled by multiple annotators. Then, the aggregated annotations are treated as ground truth labels \cite{socher2013recursive,bowman2015large}. This largely reduces the variance of the prediction \cite{nowak2010reliable,kruger2014axiomatic}, however, such aggregated results are still not oracle, and their difference to oracle remains unclear. In our paper, we proves that the accuracy on aggregated human prediction, as ground truth, could be considered as a special case of the lower bound of oracle accuracy for machine learning models.
%Considering that human could be careless and subjective, the accuracy on human annotation may not reflect the  %Many recent advance of machine learning models, especially deep learning models, are based on such measure. However, current accuracy is measuring the mode
On the other hand, much work considers the reliability of collected data, by providing the agreement scores between annotators~\cite{landis1977measurement}. Statistical measures for the reliability of inter-annotator agreement~\cite{gwet2010handbook}, such as Cohen's Kappa~\cite{pontius2011death} and Fleiss' Kappa~\cite{fleiss1971measuring}, are normally based on the raw agreement ratio.  However, the agreement between annotators does not obviously reflect the oracle accuracy; \textit{e.g.} identical predictions from two annotators does not mean they are both oracles. In our paper, we prove that observed agreement between all annotators could serve as an upper bound for the average oracle accuracy of those annotators. Overall, we propose a theory for comparing the oracle accuracy of human annotators and machine learning models, by connecting the aforementioned bounds. 
%\xqk{Do we need to add concentration theory as related work?}
%\xqk{Agreement between annotators, which will grant the reliability of annotators.}
%\xqk{Comparison between humanhuman performance and human-model performance in existing work.}

The discovery that models can predict better than human experts dates back at least to the seminal and controversial work of \citet{Meehl54clinicalversus}, which compared \textit{ad hoc} predictions based on subjective and informal information to those based on simple linear models with a (typically small) number of relevant numeric attributes. Subsequent work found that one may even train such a model to mimic the predictions made by the experts (rather than an oracle), and yet still maintain superior out of sample performance \cite{goldberg70}. The comparison of human and algorithmic decision making remains an active topic of psychology research \cite{kahnemannoise}.

%% file: sec2-theorem.tex
%!TEX root = main.tex

\section{Evaluation Theory}
In this section, we present the theory for comparing the oracle accuracy for classification tasks between human annotators and machine learning models.

\subsection{Problem Statement}

We are given $K$ labels crowd sourced from $K$ human annotators, $\{ \ell_i \}_{i=1}^K$, along some labels from a model $\ellmodel$.
We denote $\ell_a^{(n)}$ and $\ellmodel^{(n)}$ the label assigned by annotator $a$ and model $\mathcal M$ to the $n$-th data point, for $n=1,2,\dots,N$.
We observe the ratio of matched labels $\mathbb{P}(\ell_a = \ell_{b})$ for all of a pairs of annotators $a$ and $b$. Denote by $\ellaverage$ the label of the ``average'' human annotator which we define as the label obtained by selecting one of the $K$ human annotators uniformly at random. We seek to formally compare the oracle accuracy of the average human, $\mathbb{P}(\ellaverage = \elloracle)$, with that of the machine learning model, $\mathbb{P}(\ellmodel = \elloracle)$, where $\elloracle$ is the unobserved oracle label. Denote by $\ellaggregate$ the label obtained by aggregating (say, by majority voting) the $K$ human annotators' labels.
Our work distinguishes between the oracle accuracy $\mathbb{P}(\ellmodel = \elloracle)$ and the agreement with human annotations $\mathbb{P}(\ellmodel = \ell_\mathcal{G})$, although these two concepts have been confounded in many previous applications and benchmarks.

\subsection{An Upper Bound for the Average Annotator Performance}
The oracle accuracy of the average annotator $\ellaverage$ follows the definition of the previous section, and conveniently equals the average of the oracle accuracy of each annotator, \textit{i.e.}
\begin{equation}
     \label{eqn:averageannotator}
    \mathbb{P}(\ellaverage = \elloracle) = \frac{1}{K}\sum_{i=1}^{K} \mathbb{P}(\ell_i = \elloracle).
\end{equation}
By introducing an assumption, also discussed in Section~\ref{subsec:result_discussion}, we may bound the above quantity.
\begin{theorem}[Average Performance Upper Bound] \label{thm:basicupperbound}
Assume annotators are positively correlated, namely $\mathbb{P}(\ell_i = \elloracle|\ell_j = \elloracle) \geq \mathbb{P}(\ell_i = \elloracle)$.
Then, the upper bound of averaged annotator accuracy with respect to the oracle is
\begin{equation}
\label{eqn:basicupperbound}
    \mathbb{P}(\ellaverage = \elloracle) \leq \mathcal{U} \triangleq \sqrt{ \frac{1}{K^2} \sum_{i=1}^{K}\sum_{j=1}^{K}\mathbb{P}(\ell_i = \ell_j)}.
\end{equation}
\end{theorem}

\begin{proof}\  \\
%\xqk{I provide the full version of the proof here. We may move this to appendix and leave a trimmed version in paper.}\\
For $i \neq j$ and $i,j\in \{1,\cdots, K\}$, we have
\begin{align}  \nonumber
\mathbb{P}(\ell_i = \ell_j) = &\mathbb{P}(\ell_i = \ell_j|\ell_j = \elloracle)\mathbb{P}(\ell_j = \elloracle)+ \\ \nonumber
&\mathbb{P}(\ell_i = \ell_j|\ell_j \neq \elloracle)\mathbb{P}(\ell_j \neq \elloracle)\\  \nonumber
\geq & \mathbb{P}(\ell_i = \ell_j|\ell_j = \elloracle)\mathbb{P}(\ell_j = \elloracle) \\  \nonumber
= & \mathbb{P}(\ell_i = \elloracle|\ell_j = \elloracle)\mathbb{P}(\ell_j = \elloracle) \\ 
\geq & \mathbb{P}(\ell_i = \elloracle)\mathbb{P}(\ell_j = \elloracle). \label{eq:avg_upper_bound_neq}
\end{align}
While for $i=j$, we have $\mathbb{P}(\ell_i = \ell_j)=1$. Therefore,
\begin{equation}
    \mathbb{P}(\ell_i = \ell_j) \geq \mathbb{P}(\ell_i = \elloracle)\mathbb{P}(\ell_j = \elloracle). \label{eq:avg_upper_bound_eq}
\end{equation}
Then, combining \eqref{eq:avg_upper_bound_neq} and \eqref{eq:avg_upper_bound_eq},
\begin{align}
    \mathbb{P}(\ellaverage = \elloracle)^2= &\frac{1}{K^2}\sum_{i=1}^{K} \mathbb{P}(\ell_i = \elloracle)\sum_{j=1}^{K} \mathbb{P}(\ell_j = \elloracle) \\
    \leq &\frac{1}{K^2} \sum_{i=1}^{K} \sum_{j=1}^{K}\mathbb{P}(\ell_i = \ell_j)\\
    \mathbb{P}(\ellaverage = \elloracle) \leq & \sqrt{\frac{1}{K^2} \sum_{i=1}^{K}\sum_{j=1}^{K}\mathbb{P}(\ell_i = \ell_j)}. \label{eq:upper_bound}
\end{align}
% $\QED$
\end{proof}

We observe that $\mathbb P(\ell_i = \ell_j)$ is overestimated as $1$ when $i=j$, but that the total overestimation to $\mathcal{U}^2$ is less or equal to $1/K$ ($K$ out of $K^2$ terms), and that the influence will reduce and converge to zero as $K\rightarrow \infty$. 
To calibrate the overestimation, we introduce an empirically approximated upper bound $\mathcal{U}^{(e)}$. In contrast, $\mathcal{U}$ in \eqref{eqn:basicupperbound} is also noted as theoretical upper bound, $\mathcal{U}^{(t)}$.
\begin{definition}
    The \textbf{empirically approximated upper bound}, %$\mathcal{U}^{(e)}$,
\begin{equation}
    \mathcal{U}^{(e)} \triangleq \sqrt{\frac{1}{K(K-1)} \sum_{i=1}^{K}\sum_{\substack{j=1\\i\neq j}}^{K}\mathbb{P}(\ell_i = \ell_j)}. \label{eq:upper_bound_empirical}
\end{equation}
\end{definition}
\begin{lemma}[Convergence of $\mathcal{U}^{(e)}$]\label{lem:ue}
 Assume that $\mathbb{P}(\ell_i = \ell_j) \geq 1/N_c$, where $N_c$ is number of classes.
 The approximated upper bound $\mathcal{U}^{(e)}$ satisfies
\begin{equation}
    \lim_{K\rightarrow +\infty} \mathcal{U}/\mathcal{U}^{(e)} =1. \label{eq:empirical_approx}
\end{equation}
Therefore,  with large $K$, $\mathcal{U}^{(e)}$ converges to $\mathcal{U}$ or $\mathcal{U}^{(t)}$.
\end{lemma}
We provide a detailed proof of this Lemma in Appendix~\ref{appendix:proof}. Some empirical evidences for the convergence of the $\mathcal{U}^{(e)}$ to $\mathcal{U}^{(t)}$ are demonstrated in Section~\ref{subsec:result_discussion}.

\subsection{A Lower Bound for Model Performance}

For our next result, we introduce another assumption, also discussed in Section~\ref{subsec:result_discussion}. Given two predicted labels $\ell_a$ and $\ell_b$, we assume that $\ell_b$ is reasonably predictive even on those instances that $a$ gets wrong, as per
\begin{theorem}[Performance Lower Bound]
\label{thm:basiclowerbound}
Assume that for any incorrect label $\ellincorrect \neq \elloracle$,
\begin{equation}
    \mathbb{P}(\ell_b = \elloracle|\ell_a \neq \elloracle) \geq \mathbb{P}(\ell_b = \ellincorrect|\ell_a \neq \elloracle). \label{eq:assumption_low}
\end{equation}
Then, the lower bound for the oracle accuracy of $\ell_b$ is
\begin{equation}
    \mathcal{L} \triangleq \mathbb{P}(\ell_a = \ell_b) \leq \mathbb{P}(\ell_b = \elloracle). \label{eq:lower_bound}
\end{equation}
\end{theorem}
\begin{proof}
\begin{align}
    \mathbb{P}(\ell_a = \ell_b) = &\mathbb{P}(\ell_b = \ell_a|\ell_a \neq \elloracle) \mathbb{P}(\ell_a \neq \elloracle) + \nonumber\\
    &\mathbb{P}(\ell_b = \ell_a|\ell_a = \elloracle) \mathbb{P}(\ell_a = \elloracle) \nonumber\\
\leq &\mathbb{P}(\ell_b = \elloracle|\ell_a \neq \elloracle) \mathbb{P}(\ell_a \neq \elloracle) + \nonumber\\
    &\mathbb{P}(\ell_b = \ell_a|\ell_a = \elloracle) \mathbb{P}(\ell_a = \elloracle) \nonumber\\
= &\mathbb{P}(\ell_b = \elloracle|\ell_a \neq \elloracle) \mathbb{P}(\ell_a \neq \elloracle) + \nonumber\\
    &\mathbb{P}(\ell_b = \elloracle|\ell_a = \elloracle) \mathbb{P}(\ell_a = \elloracle) \nonumber\\
    =&\mathbb{P}(\ell_b = \elloracle).
\end{align}
\end{proof}
%Both models could be human or model in our theory, as long as they satisfy our assumptions. In later discussion of this paper, we consider (virtual) aggregated annotators as model $a$ and machine learning models (classifiers) as model $b$.
%\xqk{We will consider aggregated virtual annotator as a and machine learning model as b, in our later discussion.}
% $\QED$
 
% This result gives a lever into bounding the model accuracy, but requires a relatively strong reference annotator to obtain a useful bound.
%\cw{I incorporated your comment here, Q:} \xqk{May need to change to $\mathcal G$ here.} \cw{Yes, that's right thanks!} \xqk{How about the new description?} \cw{Hmm, no need to mention $\mathcal K$ I think; $\mathcal K$ the same strength as any one annotator (assuming they are equally strong), so it would not help. I say keep it simple and mention just majority voting. I've changed it and left your text in the comments. $\ellaggregate$ is defined in the setup section. NOTE (IF YOU WISH TO CHANGE THE NOTATION) THAT I DEFINED A COMMAND FOR $\ellaggregate$.}
In practice, a more accurate $\ell_a$ gives a tighter lower bound for $\ell_b$, and so we employ the aggregated human annotations for the former (letting $\ell_a=\ellaggregate$) to calculate the lower bound of the machine learning model (letting $\ell_b=\ellmodel$), as demonstrated in Section~\ref{subsec:result_discussion}.

% In practice, a superior $\ell_a$ gives a tighter lower bound for $\ell_b$. We exploit this by aggregating human annotations into a new \textit{virtual} annotator $\mathcal{G}$\footnote{Although average annotator $\mathcal K$ could be considered as a special case of aggregated annotator $\mathcal{G}$, we consider $\mathcal{G}$ generally better than $\mathcal{K}$.} (thereby letting $\ell_a=\ell_{\mathcal{G}}$) to estimate the lower bound of a machine learning model (letting $\ell_b=\ellmodel$), as demonstrated in Section~\ref{subsec:result_discussion}.

\paragraph{Connection to common practice.} Generally, the ground truth of a benchmark corpus is constructed by aggregating multiple human annotations \cite{wang2018glue,wang2019superglue}. For example, the averaged sentiment score is used in SST \cite{socher2013recursive} and majority of votes in SNLI \cite{bowman2015large}. Then, the aggregated annotations are treated as ground truth to calculate accuracy. Under this setting, the accuracy on the (aggregated) human ground truth may be viewed as a special case of our lower bound.

% We consider the aggregated results of multiple annotation as from a \textit{virtual} annotator that probably works better than any of the individual annotators~\cite{TODO_boosting_papers}. By utilizing the predictions of the aggregated (virtual) model in our theorem, the accuracy on human ground truth, used in previous work~\cite{socher2013recursive,bowman2015large}, could be viewed as a specific lower bound of the accuracy with regard to oracle.

\subsection{Finite Sample Analysis}

The results above assume that the agreement probabilities are known; we now connect with the finite sample case where those probabilities are estimated empirically. We begin with a standard concentration inequality (see \textit{e.g.} \cite[\S~2.6]{boucheron2013concentration}),
\begin{theorem}[Hoeffding's Inequality]
\label{thm:hoeffding}
    Let $X_1, \dots, X_N$ be independent random variables with finite variance such that $\mathbb{P}(X_n\in [\alpha, \beta])=1$, for all $1 \leq n \leq N$. Let
    $$\overline{X} \triangleq \frac{1}{N} \sum_{n=1}^N X_n,$$ 
    then, for any $t > 0$,
    \begin{align}
        \label{eqn:hoeffding} \nonumber
        \mathbb P(\overline{X} - \mathbb{E}[\overline{X}] \geq +t) \leq \exp\left( -\frac{2 N t^2}{(\alpha - \beta)^2} \right),\\
        \mathbb P(\overline{X} - \mathbb{E}[\overline{X}] \leq -t) \leq \exp\left( -\frac{2 N t^2}{(\alpha - \beta)^2} \right).
    \end{align}
\end{theorem}

\noindent Combining this with Thereom~\ref{thm:basicupperbound} we obtain the following.
\begin{theorem}[Sample Average Performance Upper Bound]%\xqk{Average Sample Performance Upper Bound?}
\label{thm:finitebasicupperbound}
Take the assumptions of Theorem~\ref{thm:basicupperbound}, and let %\cw{I think it's okay but you can add something if you want. It is called the Iverson bracket.}
\begin{align}
    \label{eqn:pn}
    \mathbb{P}^{(N)}(\ell_i = \ell_j)=\frac 1 N \sum_{n=1}^N\left[\ell^{(n)}_i=\ell^{(n)}_j\right]
\end{align}
be the empirical agreement ratio\footnote{Here $[\cdot]$ is the Iverson bracket.}.
Define
\begin{align}
\label{eqn:upperdelta}
\delta_u=\exp(-2 N t_u^2).
\end{align}
With probability at least $1-\delta_u$, for any $t_u>0$,
\begin{align}
\label{eqn:finitebasicupperbound}
     \mathbb{P}(\ellaverage = \elloracle)\leq\sqrt{t_u+\frac{1}{K^2} \sum_{i=1}^{K}\sum_{j=1}^{K}\mathbb{P}^{(N)}(\ell_i = \ell_j)}.
\end{align}
\end{theorem}

\begin{proof}
We apply Theorem~\ref{thm:hoeffding} with
%\xqk{I rough know why we need 1/N here. Do we need some intuitive explanations?}
\begin{align}
    \label{eqn:xnsampleupper}
    X_n = \frac{1}{K^2} \sum_{i=1}^{K}\sum_{j=1}^{K} \left[\ell^{(n)}_i=\ell^{(n)}_j\right],
\end{align}
obtaining $X_n \in [0,1]$, \textit{i.e.} $\alpha=0$, and $\beta=1$. Let 
\begin{align}
    \label{eqn:un}
    \mathcal{U}_{N} \triangleq \sqrt{\frac{1}{K^2} \sum_{i=1}^{K}\sum_{j=1}^{K}\mathbb{P}^{(N)}(\ell_i = \ell_j)}.
\end{align}
Our choice \eqref{eqn:xnsampleupper} of $X_n$ implies $\mathcal{U}_{N}^2=\overline{X}$ and $\mathcal{U}^2=\mathbb{E}[\overline{X}]$,
and so by \eqref{eqn:hoeffding},
    \begin{align}
    \label{eqn:middleinequalitya}
        \mathbb P\left(\sqrt{t_u+\mathcal{U}_{N}^2} \leq \mathcal{U}\right) \leq \delta_u.
    \end{align}
Rewrite \eqref{eqn:basicupperbound} as%\xqk{Note that I add a bar to represent `average' in Eq2.}
    \begin{align}
    \label{eqn:basicupperboundnew}
            \mathbb{P}(\ellaverage = \elloracle) \leq \mathcal{U},
    \end{align}
which implies
    \begin{align}
    \label{eqn:middleinequalityb}
                \mathbb P\left(\sqrt{t_u+\mathcal{U}_{N}^2} \leq \mathbb{P}(\ellaverage = \elloracle)\right)
                \leq
                \mathbb P\left(\sqrt{t_u+\mathcal{U}_{N}^2} \leq \mathcal{U}\right).
    \end{align}
Combining \eqref{eqn:middleinequalitya} with \eqref{eqn:middleinequalityb} gives the result.
\end{proof}
\noindent Analagously for Theorem~\ref{thm:basiclowerbound}, we have
\begin{theorem}[Sample Performance Lower Bound]
\label{thm:finitebasiclowerbound}
Take the assumptions of Theorem~\ref{thm:basiclowerbound} along with \eqref{eqn:pn}.
Define
\begin{align}
\label{eqn:lowerdelta}
\delta_l=\exp(-2 N t_l^2).
\end{align}
With probability at least $1-\delta_l$, for any $t_l>0$,
\begin{align}
\label{eqn:finitebasiclowerbound}
\mathbb P^{(N)}(\ell_a=\ell_b)
\leq
t_l + \mathbb P(\ell_b=\elloracle).
\end{align}
\end{theorem}

\begin{proof}
We apply Theorem~\ref{thm:hoeffding} with
\begin{align}
    \label{eqn:xnsamplelower}
    X_n = \left[\ell^{(n)}_a=\ell^{(n)}_b\right],
\end{align}
obtaining $X_n \in [0,1]$, \textit{i.e.} $\alpha=0$, and $\beta=1$. Let
    \begin{align}
    \label{eqn:ln}
        \mathcal{L}_N \triangleq \mathbb P^{(N)}(\ell_a=\ell_b).
    \end{align}
Now \eqref{eqn:xnsamplelower} implies $\mathcal{L}_{N}=\overline{X}$ and $\mathcal{L}=\mathbb P(\ell_a=\ell_b)= \mathbb{E}[\overline{X}]$,
    \begin{align}
    \label{eqn:lowermiddleinequalitya}
        \mathbb P\left(\mathcal{L}_N-t_l\geq \mathcal{L} \right) \leq \delta_l.
    \end{align}
Recall \eqref{eq:lower_bound}, $\mathbb{P}(\ell_a = \ell_b) \leq \mathbb{P}(\ell_b = \elloracle)$, which implies
% \begin{align}
%     \label{eqn:lowermiddleinequalityb}
%     & \mathbb P\left(\mathbb P^{(N)}(\ell_a=\ell_b)-t_l\geq \mathbb P(\ell_b=\elloracle)\right)\\
%     & \leq
%     \mathbb P\left(\mathbb P^{(N)}(\ell_a=\ell_b)-t_l\geq \mathbb P(\ell_a=\ell_b)\right).\nonumber
% \end{align}
\begin{align}
    \label{eqn:lowermiddleinequalityb}
    \mathbb P\left(\mathcal{L}_N-t_l\geq \mathbb P(\ell_b=\elloracle)\right) \leq
    \mathbb P\left(\mathcal{L}_N-t_l\geq \mathcal{L}\right).
\end{align}
Combining \eqref{eqn:lowermiddleinequalitya} with \eqref{eqn:lowermiddleinequalityb} gives the result.
\end{proof}

% On the other hand, the following lemma gives a simple numerical recipe to guide XXX \cw{Can we make an experiment for this?}. 
% \begin{lemma}
% \label{thm:epsilonparam}
% By fixing the confidence level $\gamma = \log \mathbb P(S \geq t)$; then we may rewrite the condition of \eqref{eqn:bernoullibouond} as
% \begin{align}
%     \epsilon 
%     & \leq 
%     \frac{- 2 / 3 \gamma
%     +  \sqrt{(2 / 3 \gamma)^2 - 4 N \gamma 2\hat \rho}
%     }{2N}.
% \end{align}  
% \end{lemma}
% \begin{proof}
% The condition is
% \begin{align}
%     \gamma 
%     & \geq -\frac{\epsilon^2 N}{2(\hat \rho+\epsilon/3)} \\
%     \gamma 2(\hat \rho+\epsilon/3)
%     & \geq -\epsilon^2 N \\
%     \epsilon^2 N + \epsilon 2 / 3 \gamma + \gamma 2\hat \rho
%     & \geq 0;
% \end{align}  
%  considering the solutions of the quadratic gives the result.
% \end{proof}

\subsection{Detecting and Certifying Superhuman Models} %\xqk{Add Certifying to section title?}\cw{I would switch to that and remove detecting, yep}.
We propose a procedure to discover potentially superhuman models based on our theorems.
\begin{itemize}
    \item Calculate the upper bound of the average oracle accuracy of human annotators, $\mathcal{U}_N$, with $N$ samples;
    \item Calculate the lower bound of the model oracle accuracy $\mathcal{L}_N$ using aggregated human annotations as the reference\footnote{We demonstrate that aggregating the predictions by voting and weighted averaging are effective in improving our bounds. We emphasize however that the aggregated predictions need not be perfect, as we do not assume that this aggregation yields an oracle.}, with $N$ samples;
    \item Check whether the finite sample margin $\mathcal{L}_N-\mathcal{U}_N$ is larger than zero;
    \item Give proper estimation of $t_u$ and $t_l$ and calculate a confidence score of $\mathbb P(\mathcal{L}-\mathcal{U} \geq 0)$.
\end{itemize}

Generally, larger margin indicates higher confidence of the out-performance. To formally check confidence for the aforementioned margin we provide
\begin{theorem}[Confidence of Out-Performance]
    \label{thm:marginoutperformancetwodeltas} %\xqk{Confidence of Out-Performance?} \cw{Sound fine to me!}
    Assume an annotator pool with agreement statistic $\mathcal U_N$ of \eqref{eqn:un}, and an agreement statistic between model and aggregated annotations $\mathcal L_N$ of \eqref{eqn:ln}.
    %an average annotator $\ell_a$ defined in \eqref{eqn:averageannotator}, and a reference annotator with agreement statistic $\mathbb P^{(N)}(\ell_a=\ell_b)$ of \eqref{eqn:pn}. Define
    If $\mathcal L_N > \mathcal U_N$ then for all $\tau \geq 0$, $t_u \geq 0$ and $t_l \geq 0$ that satisfy
    \begin{equation}
        \label{eqn:margin_assumption}
        \mathcal L_N - t_l - \sqrt{t_u + \mathcal U_N^2} = \tau,
    \end{equation}
    with probability at least $1-\delta_u-\delta_l$, the model oracle accuracy exceeds that of the average annotator by $\tau$, \textit{i.e.}
    % \xqk{I use $t\rightarrow\tau$ later.}
    \begin{align}
        \label{eqn:marginslack}
        \mathbb{P}(\ellmodel = \elloracle) - \mathbb P(\ellaverage=\elloracle) \geq \tau,
    \end{align}
    where
    \begin{align}
        \label{eqn:delta_u}
        \delta_u & = \exp\left(-2 N t_u^2\right),
        \\
        \label{eqn:delta_l}
        \delta_l & = \exp\left(-2 N t_l^2\right).
    \end{align}
\end{theorem}
\begin{proof}
    % With $t_a$ and $t_b$ denoting the slack term $t$ of Theorem~\ref{thm:finitebasicupperbound} and Theorem~\ref{thm:finitebasiclowerbound}, respectively, we have to satisfy
    % \begin{align}
    %     \mathbb{P}(\ell_a = \elloracle) 
    %     & \leq 
    %     \sqrt { t_a + \mathcal U_N^2} \\
    %     \mathbb P^{(N)}(\ell_a=\ell_b)-t_b
    %     & \leq
    %     \mathbb P(\ell_b=\elloracle).
    % \end{align}
   Recall Theorem~\ref{thm:finitebasicupperbound} and Theorem~\ref{thm:finitebasiclowerbound},
%   \begin{align}
%   \nonumber
%   \mathbb P\left(\sqrt{t_u+\mathcal{U}_{N}^2} \leq \mathbb{P}(\ell_a = \elloracle)\right)
%   &\leq
%   \mathbb P\left(\sqrt{t_u+\mathcal{U}_{N}^2} \leq \mathcal{U}\right)
%   \leq \delta_u
%   \\
%   \nonumber
%   \mathbb P\left(\mathcal{L}_N-t_l\geq \mathbb P(\ell_b=\elloracle)\right)
%   &\leq
%   \mathbb P\left(\mathcal{L}_N-t_l\geq \mathcal{L}\right)
%   \leq \delta_l
%   \end{align}
   \begin{align}
   \nonumber
   \mathbb P\left(\sqrt{t_u+\mathcal{U}_{N}^2} \leq \mathbb{P}(\ellaverage = \elloracle)\right)
   &\leq \delta_u
   \\
   \nonumber
   \mathbb P\left(\mathcal{L}_N-t_l\geq \mathbb P(\ellmodel=\elloracle)\right)
   &\leq \delta_l.
   \end{align}
   Then, we have
   \begin{align}
   & \mathbb P\left(\mathbb P(\ellmodel=\elloracle) - \mathbb{P}(\ellaverage = \elloracle) \geq \tau\right) \nonumber\\
   = & \nonumber\mathbb P\left(\mathbb P(\ellmodel=\elloracle) - \mathbb{P}(\ellaverage = \elloracle) \geq \mathcal L_N - t_l - \sqrt{t_u + \mathcal U_N^2}\right)  \\
   \geq & \nonumber\mathbb P\left(\mathbb P(\ellmodel=\elloracle) \geq \mathcal L_N - t_l \cap  \mathbb{P}(\ellaverage = \elloracle) \leq  \sqrt{t_u + \mathcal U_N^2}\right)  \\
   & =  \nonumber 1-\mathbb P\left(\mathbb P(\ellmodel=\elloracle) \leq \mathcal L_N - t_l \cup  \mathbb{P}(\ellaverage = \elloracle) \geq  \sqrt{t_u + \mathcal U_N^2}\right) \\
   \geq & \nonumber 1-\mathbb P\left(\mathbb P(\ellmodel=\elloracle) \leq \mathcal L_N - t_l) \right) - \mathbb P\left(  \mathbb{P}(\ellaverage = \elloracle) \geq  \sqrt{t_u + \mathcal U_N^2}\right) \\
   \geq & 1-\delta_l-\delta_u.
% next two lines are saved 
%   \geq & \nonumber\mathbb P\left(\mathbb P(\ellmodel=\elloracle) \geq \mathcal{L}_N - t_l\right) \cdot \mathbb P\left(\mathbb{P}(\ellaverage = \elloracle) \leq  \sqrt{t_u + \mathcal U_N^2}\right)\\
%   \geq & (1-\delta_l)(1-\delta_u).
   \end{align}
%   \QED
\end{proof}

\paragraph{Confidence Score Estimation.}
The above theorem suggests the confidence score
\begin{align}
    S=1-\delta_l-\delta_u,
\end{align}
and we need now only choose the free constants $t_l, t_u$ and $\tau$, which it depends on.
Recall \eqref{eqn:margin_assumption},
    \begin{equation}
     \tau = \left(\mathcal L_N-t_l\right)-\sqrt { t_u + \mathcal U_N^2},
     \end{equation}
and remove one degree of freedom parameterise in $t_u$ as
    \begin{equation}
     t_l(t_u,\tau) = \mathcal L_N-\tau-\sqrt { t_u + \mathcal U_N^2}.
    \end{equation}
We are interested in $\mathbb{P}(\mathcal L - \mathcal{U} \geq 0)$ and so we may set $\tau=0$. We offer two alternatives for selecting $t_u$ and $t_l$.

\noindent\textbf{Algorithm 1 (Heuristic Margin Separation, HMS).}
We assign half of the margin to $t_u$,
\begin{equation}
    t_u = \frac{\mathcal L_N-\mathcal U_N}{2}. 
\end{equation}
Then, with $\tau=0$ we calculate the corresponding
\begin{equation}
    t_l = \mathcal L_N-\sqrt { \frac{\mathcal L_N-\mathcal U_N}{2} + \mathcal U_N^2},
\end{equation}
% The proof for such assignment is provided in Appendix~\ref{appendix:proof}.
and compute the heuristic confidence score $S$. 

\noindent\textbf{Algorithm 2 (Optimal Margin Separation, OMS).}\ \\
For an locally (in $t_u$) optimal confidence score, we perform gradient ascent~\cite{lemarechal2012cauchy} on $S(t_u)$, where
\begin{align}
     &S(t_u) = 1-\delta(t_u)-\delta(t_l(t_u,0)), %,\\
     % &t_l(t_u,0) = \mathcal L_N-\sqrt { t_u + \mathcal U_N^2},
\end{align} 
with $t_u$ is initialized as $(\mathcal L_N-\mathcal U_N)/2$ before optimization\footnote{For all OMS experiments, we set learning rate 1e-4, and iterate 100 times. We will publish our code upon acceptance.}.
%We numerically search over $t_u \in \text{Grid}(0, )$ \xqk{We may need a more intuitive grid function. The current one looks correct but not sure how to explain it in the paper.}

%\cw{ We may verify that \eqref{eqn:tb} and \eqref{eqn:ta} are chosen i) for $t=0$, to make the lower and upper bounds above equal to one another, and in particular to the mid-point of the $t_a=t_b=0$ bounds, and ii) for $t>0$ to move the lower (upper) bound up (down) by $t/2$, thereby equally sharing the ``slack'' $t$ of \eqref{eqn:marginslack} between Theorem~\ref{eqn:finitebasicupperbound} and Theorem~\ref{eqn:finitebasiclowerbound}. By construction the overall slack is therefore $t/2+t/2=t$ as required by \eqref{eqn:marginslack}, so that Theorem~\ref{thm:finitebasicupperbound} and Theorem~\ref{thm:finitebasiclowerbound} may be applied to give the result.}

%% file: sec3-exp.tex
%!TEX root = main.tex

\section{Experiments and Discussion}
Previously, we introduced a new theory for analyzing the oracle accuracy of set of classifiers using observed agreements between them. In this section, we demonstrate our theory on several classification tasks, to demonstrate the utility of the theory and reliability of the associated assumptions.%we will describe our experimental settings and discuss on the utility of our theorem.

\subsection{Experimental Setup}
%\xqk{We will discuss tasks, with artificial oracle rules, and two real-word tasks, sentiment classification and natural language inference.}
%\xqk{To CC, could you describe the datasets and experiments.}
We first consider two classification tasks with oracle labels generated by rules. Given the oracle predictions, we are able to empirically validate the assumptions for our theorems and observe the convergence of the bounds. Then, we apply our theory on two real-world classification tasks and demonstrate that some existing state-of-the-art models have potentially achieved better performance than the (averaged) performance of the human annotators.
% The selection of the datasets intends to cover both computer vision and natural language processing tasks.

\paragraph{Classification tasks with oracle rules.}
To validate the correctness of our theory, we collect datasets with observable oracle labels. We construct two visual cognitive tasks, \taskcolor and \taskshape, with explicit unambiguous rules to acquire oracle labels, as follows:
\begin{itemize}
    \item \taskcolor: the oracle selects the most frequently occuring color of the objects in a given image.
    \item \taskshape: the oracle selects the most frequent occurring shape of the objects in a given image.
\end{itemize}
For both tasks, the size of the objects is ignored. As illustrated in Figure~\ref{fig:task_examples}, we vary three colors (\textit{Red}, \textit{Blue} and \textit{Yellow}) and five shapes (\textit{Triangle}, \textit{Square}, \textit{Pentagon}, \textit{Hexagon} and \textit{Circle}) for the two tasks, respectively.
\input{fig_task_example}

For each task, we generated 100 images and recruited 10 annotators from the \textit{Amazon Mechanical Turk}\footnote{\url{https://www.mturk.com}} to label them. 
Each randomly generated example includes 20 to 40 objects.
We enforce that no objects overlap more than 70\% with all others, and that there is only one class with the highest count, to ensure uniqueness of the oracle label.
% We constrain the total number of objects from 20 to 40 and there is no draw for the class with the highest count. 
The oracle number of the colors and shapes are recorded to generate oracle labels of the examples.
More details about annotation interfaces and guidelines are provided in Appendix~\ref{appendix:AMT}.

\paragraph{Real-World Classification Tasks.}
We analyze the performance of human annotators and machine learning models on two real-world NLP tasks, namely sentiment classification and natural language inference (NLI). We use the Stanford Sentiment Treebank (\textbf{SST})~\cite{socher2013recursive} for sentiment classification. The sentiment labels are mapped into two classes (SST-2)\footnote{Samples with overall neutral scores are excluded as in \cite{tai2015improved}.} or five classes (SST-5), \textit{very negative}($[0, 0.2]$), \textit{negative} ($(0.2, 0.4]$), \textit{neutral}($(0.4, 0.6]$), \textit{positive} ($(0.6, 0.8]$), and \textit{very negative} ($(0.8, 1.0]$).
We use the Stanford Natural Language Inference (\textbf{SNLI}) corpus~\cite{bowman2015large} for NLI. All samples are classified by five annotators into three categories, \textit{i.e.} \textit{Contradiction}~(C), \textit{Entailment}~(E), and \textit{Neutral}~(N).
More details of the datasets are reported in Table~\ref{tab:nlp_datasets}.
In the latter part of this section, we only report the estimated upper bounds on test sets, as we intend to compare them with the performance of machine learning models generally evaluated on test sets.
\input{tab-nlp-datasets}

% \input{fig_bounds}
\input{fig_bounds2}

\paragraph{Machine Learning Models.}
% \xqk{May need proof-reading.}\cw{done}
For both of the classification tasks with known oracles, we treat them as detection tasks and train YOLOv3 models \cite{redmon2018yolov3} for them. The input image resolution is 608 $\times$ 608 and we use the proposed Darknet-53 as the backbone feature extractor. For comparison, we train two models, a strong model and a weak model, on 512 and 128 randomly generated examples, respectively. All models are trained for a maximum of 200 epochs until convergence. During inference, the model detects the objects and we count each type of object to obtain the prediction.

We compare several representative models and their variants for real-world classification tasks, such as Recurrent Neural Networks~\cite{chen2018enhancing,zhou2015c}, Tree-based Neural Networks~\cite{mou2016natural,tai2015improved}, and Pre-trained Transformers~\cite{devlin2019bert,radfordimproving,wang2020structbert,sun2020self}.

\subsection{Results and Discussion}
\label{subsec:result_discussion}

We now conduct several experiments to validate the convergence of the bounds and the validity of the assumptions. We then demonstrate the utility of our theory by detecting superhuman models. We organize the discussion into several research questions (RQ).

\paragraph{RQ1: Will the bounds converge given more annotators?} 
We first analyze the lower bounds. We demonstrate lower bounds for strong models (s) and weak models (w) in Figure~\ref{fig:bounds} in black and blue lines respectively. Generally, \textit{i)} the lower bounds $\mathcal L_N$ are always under the oracle accuracy of corresponding models; \textit{ii)} the lower bounds grow and tend to get closer to the bounded scores given more aggregated annotators.
Then, we analyze the upper bounds. We illustrate theoretical upper bound $\mathcal{U}_N^{(t)}$ and empirically approximated upper bound $\mathcal{U}_N^{(e)}$, in comparison with average oracle accuracy of annotators $\mathbb P(\ellaverage=\elloracle)$, in Figure~\ref{fig:bounds}.
We observe that \textit{i)} both upper bounds give higher estimation than the average oracle accuracy of annotators; \textit{ii)} the margin between $\mathcal{U}_N^{(t)}$ and $\mathcal{U}_N^{(e)}$ reduce, given more annotators incorporated; \textit{iii)} $\mathcal{U}_N^{(e)}$ generally provides a tighter bound than $\mathcal{U}_N^{(t)}$, and we will use $\mathcal{U}_N^{(e)}$ as $\mathcal{U}_N$ to calculate confidence score in later discussion.

%In terms of the upper bounds, the margin between $\mathcal{U}_N^{(t)}$ and $\mathcal{U}_N^{(e)}$,  red and black lines in Figure~\ref{fig:bounds}, keeps reducing with the growth of annotator number, $K$. Both upper bounds give higher estimation than the averaged annotators oracle accuracy, as Theorem~\ref{thm:basicupperbound}. \cw{Is the following sentence valid?} 
%In terms of the lower bound, the lower bound of our machine learning model (blue line) grows given more annotators for aggregation. We attribute this to better virtual models who certificate higher lower bounds to our model, as we observe the aggregated annotator performance normally improves with more annotators incorporated (green line). , as Theorem~\ref{thm:basiclowerbound}.

%\paragraph{RQ2: Will the lower bound of the oracle accuracy of an ML model converge?}\ \\

\paragraph{RQ2: Are the assumptions of our theorems valid?} We verify the key assumptions for the upper bound of Theorem~\ref{thm:basicupperbound} and the lower bound of Theorem~\ref{thm:basiclowerbound} by computing the relevant quantities in Table~\ref{tab:test_assumption}. The two assumptions hold in our experiments, although we can only perform this analysis on the tasks with known oracle labels. For the assumptions required for our lower bound, our experiment is more conservative than the assumption, as we sum over all incorrect labels (see column 2 of Table~\ref{tab:test_assumption}.b). Despite the stricter setting, our assumption still holds on both experiments.

\input{tab-assumptions}

% ~\\
\noindent\textit{Disclaimer:} %Even though the assumptions used in our theory are proved rational in our experiments, 
while the assumptions appear reasonable, we recommend where possible to obtain a small set of oracle labels to validate the assumptions in future research.

\input{tab-nlp-tasks}
\paragraph{RQ3: How to identify a `powerful', or even superhuman, classification model?}
We first compare the $\mathcal L_N$ with $\mathcal{U}_N$ in our toy experiments, in Figure~\ref{fig:bounds}. Overall, it is more likely to observe superhuman performance given more annotators. We observe that $\mathcal L_N^{(s)}$ outperforms both $\mathcal{U}_N^{(e)}$ and $\mathcal{U}_N^{(t)}$, given more than 4 and 6 annotators for color classification and shape classification, respectively. When the model is marginally outperforming human, see weak model for color classification, we may not observe a clear superhuman performance margin, $\mathcal L_N^{(s)}$ and $\mathcal{U}_N^{(e)}$ are very close given more than 7 annotators.

For real-world classification tasks, we \textit{i)} calculate the average annotator upper bounds given multiple annotators' labels and \textit{ii)} collect model lower bounds reported in previous literature. Some results on SST and SNLI are reported in Table~\ref{tab:nlp_tasks}. We observe that pre-trained language models provide significant performance improvement on those tasks. Our theory manages to identify some of these models that potentially exceed the average human annotator performance, by comparing $\mathcal{U}^{(e)}_N$ or the even more restrictive $\mathcal{U}^{(t)}_N$.

\paragraph{RQ4: How confident are the certifications?} We calculate our confidence score for the identified outperforming models via $\mathcal{U_N}$, $\mathcal{L_N}$, $N$, and using HMS and OMS, as reported in Table~\ref{tab:sota_confidence}.
Generally, the confidence scores for SNLI models are higher than those of SST-2 because the former has test set is more than five times larger, while more recent and advanced models achieve higher confidence scores as they have larger margin of $\mathcal{L}_N - \mathcal{U}_N$.

\input{tab-confidence}

%% file: fig_task_example.tex
%!TEX root = main.tex

\begin{figure}[h]
\hfill
\begin{subfigure}[b]{0.48\linewidth}
	\includegraphics[width=1.0\linewidth]{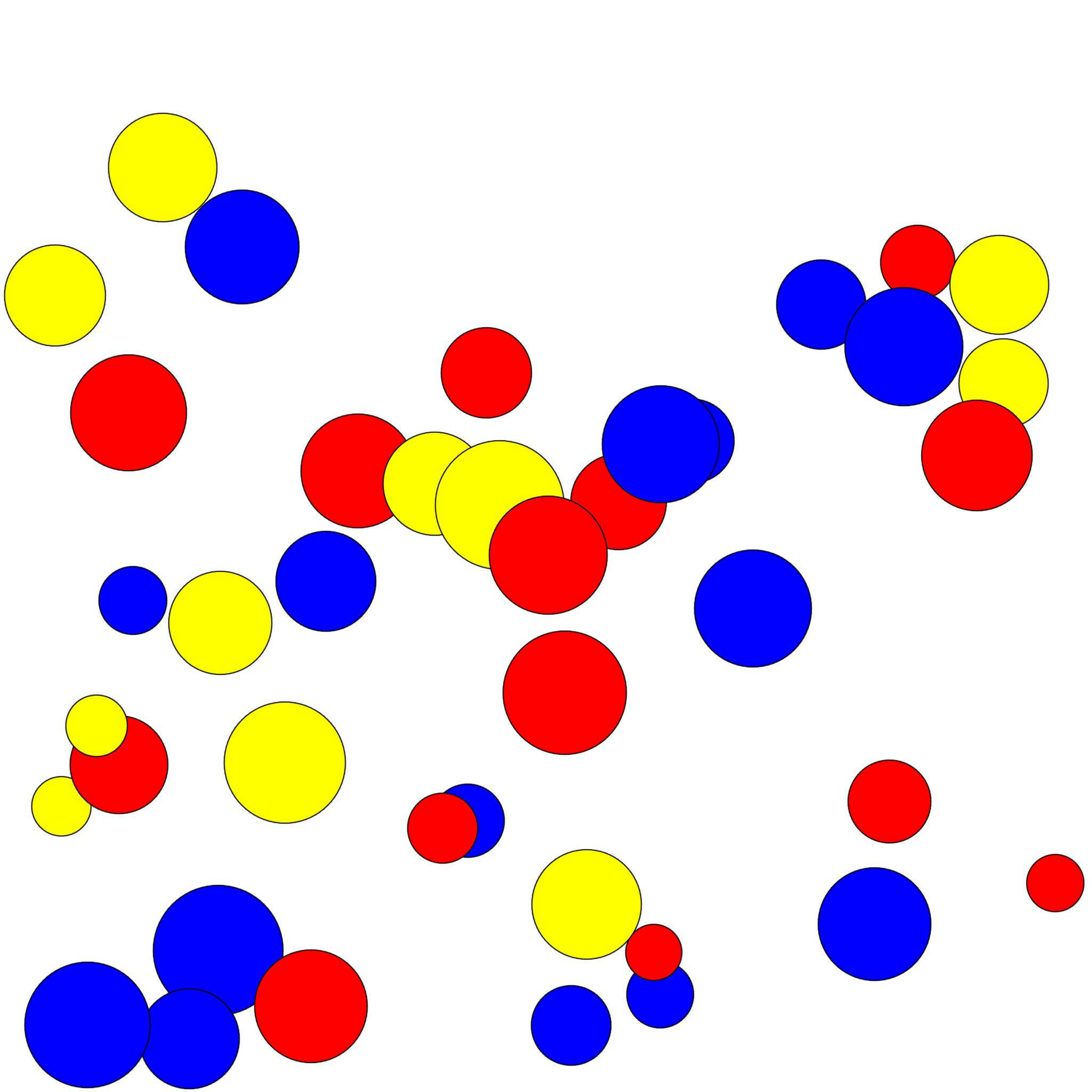}
% 	\vspace{-5mm}
	\caption{Color}
	\label{fig:task_example_color}
\end{subfigure}
 \hfill
\begin{subfigure}[b]{0.48\linewidth}
	\includegraphics[width=1.0\linewidth]{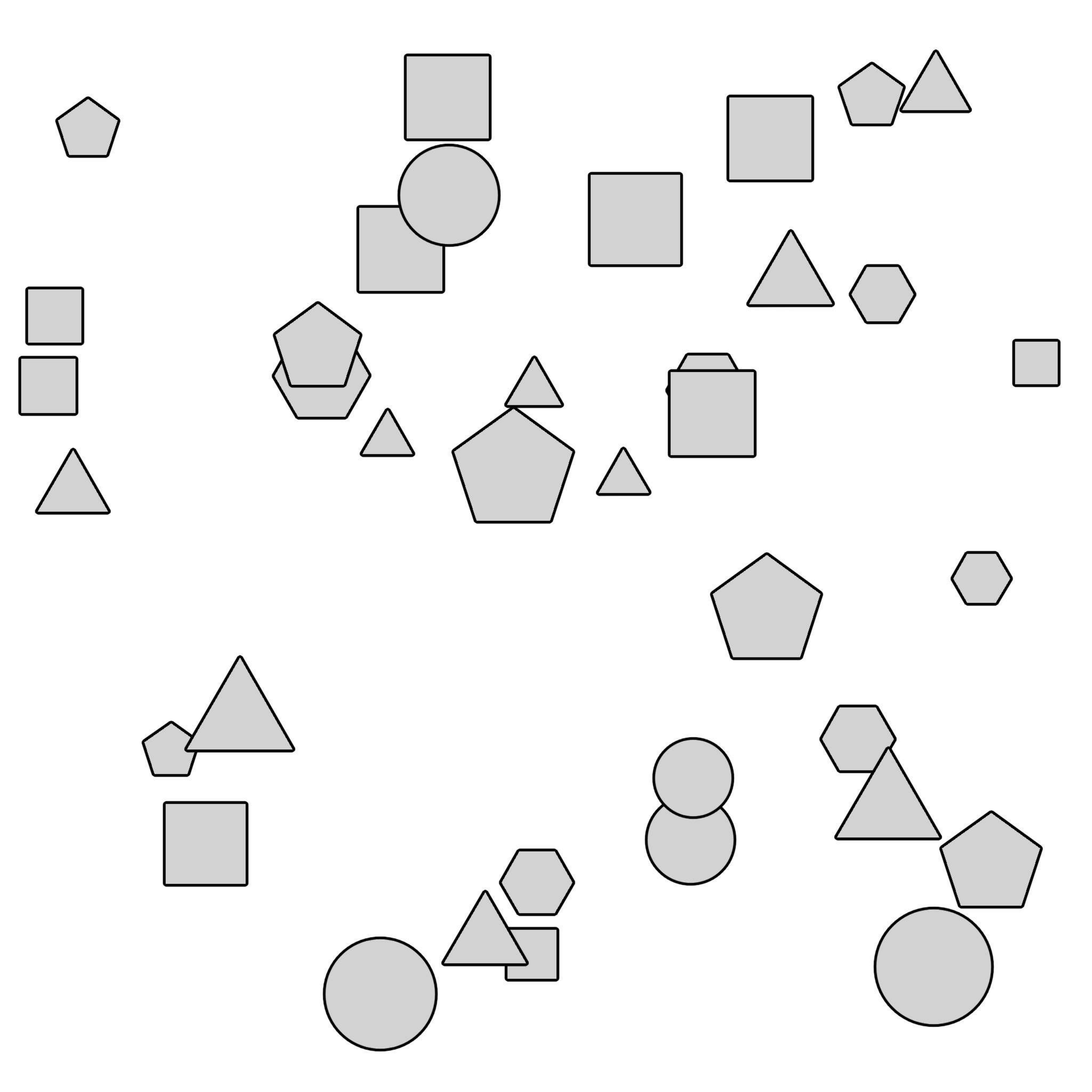}
% 	\vspace{-5mm}
	\caption{Shape}
	\label{fig:task_example_shape}
\end{subfigure}
\hfill
    % \vspace{-3mm}
	\caption{Examples of \textit{a)} \taskcolor and \textit{b)} \taskshape. Example \textit{a)} includes 40 objects of three different colors, \textit{Red} (14), \textit{Blue} (15) and \textit{Yellow} (11), with \textit{Blue} as the most frequent color and therefore the oracle label. Example \textit{b)} includes 37 objects of five different shapes, \textit{Triangle} (9), \textit{Square} (10), \textit{Pentagon} (7), \textit{Hexagon} (6) and \textit{Circle} (5), with \textit{Square} the dominant shape and oracle label.}
	\label{fig:task_examples}
    % \vspace{-3mm}
\end{figure}

%% file: tab-nlp-datasets.tex
\begin{table}[h]
    \centering
\resizebox{\linewidth}{!}{
    \begin{tabular}{lccc}
      \toprule
    
Dataset &  \#Test & \#Class & \#Annot.\\
\midrule
SST-2 \cite{socher2013recursive} & 1,821 & 2 & 3\\
SST-5 \cite{socher2013recursive} & 2,210 & 5 & 3\\
SNLI \cite{bowman2015large} & 10,000 & 3 & 5\\
    \bottomrule
 \end{tabular}
 }
% \vspace{-2mm}
    \caption{Statistics of SST and SNLI: the number of test samples, number of classes, and the number of annotators.}
    % \vspace{-2mm}
    \label{tab:nlp_datasets}
\end{table}

%% file: fig_bounds2.tex
% \begin{figure*}[t]

%     \begin{subfigure}[t]{.44\textwidth}
%     \centering
% 	\includegraphics[width=.95\linewidth]{fig/color_bound_mid.pdf}
% % 	\vspace{-3mm}
%     \caption{Color Classification}
%     \label{fig:color_bounds}
%     \end{subfigure}
%     \hfill
%     \begin{subfigure}[t]{.08\textwidth}
%     \centering
% % 	\includegraphics[width=.95\linewidth]{fig/color_bound_mid.pdf}
%     \raisebox{20mm}{
%         \includegraphics[width=.90\linewidth]{fig/legend.pdf}
%     }
% % 	\caption{}
% % 	\vspace{-3mm}
%     \end{subfigure}
%     \hfill
%     \begin{subfigure}[t]{.44\textwidth}
%     \centering
% 	\includegraphics[width=.95\linewidth]{fig/shape_bound_mid.pdf}
% % 	\vspace{-3mm}
%     \caption{Shape Classification}
%     \label{fig:shape_bounds}
%     \end{subfigure}
%     % \vspace{-2mm}
%     \caption{Comparison of sample theoretical upper bound $\mathcal{U}_N^{(t)}$ and sample empirical upper bound $\mathcal{U}_N^{(e)}$ of averaged oracle accuracy of annotators $\mathbb{P}(\ellaverage = \elloracle)$. Another comparison of sample lower bound $\mathcal L_N$ for model oracle accuracy $\mathbb{P}(\ellmodel = \elloracle)$. Relatively strong and weak models are indicated by $(s)$ and $(w)$.}
%     \label{fig:bounds}
% \end{figure*}

\begin{figure*}[t]
  \hfill
  \begin{minipage}{.43\linewidth}
    \centering
    \subcaptionbox{Color Classification}
      {\includegraphics[trim=10mm 0mm 10mm 10mm, width=\linewidth]{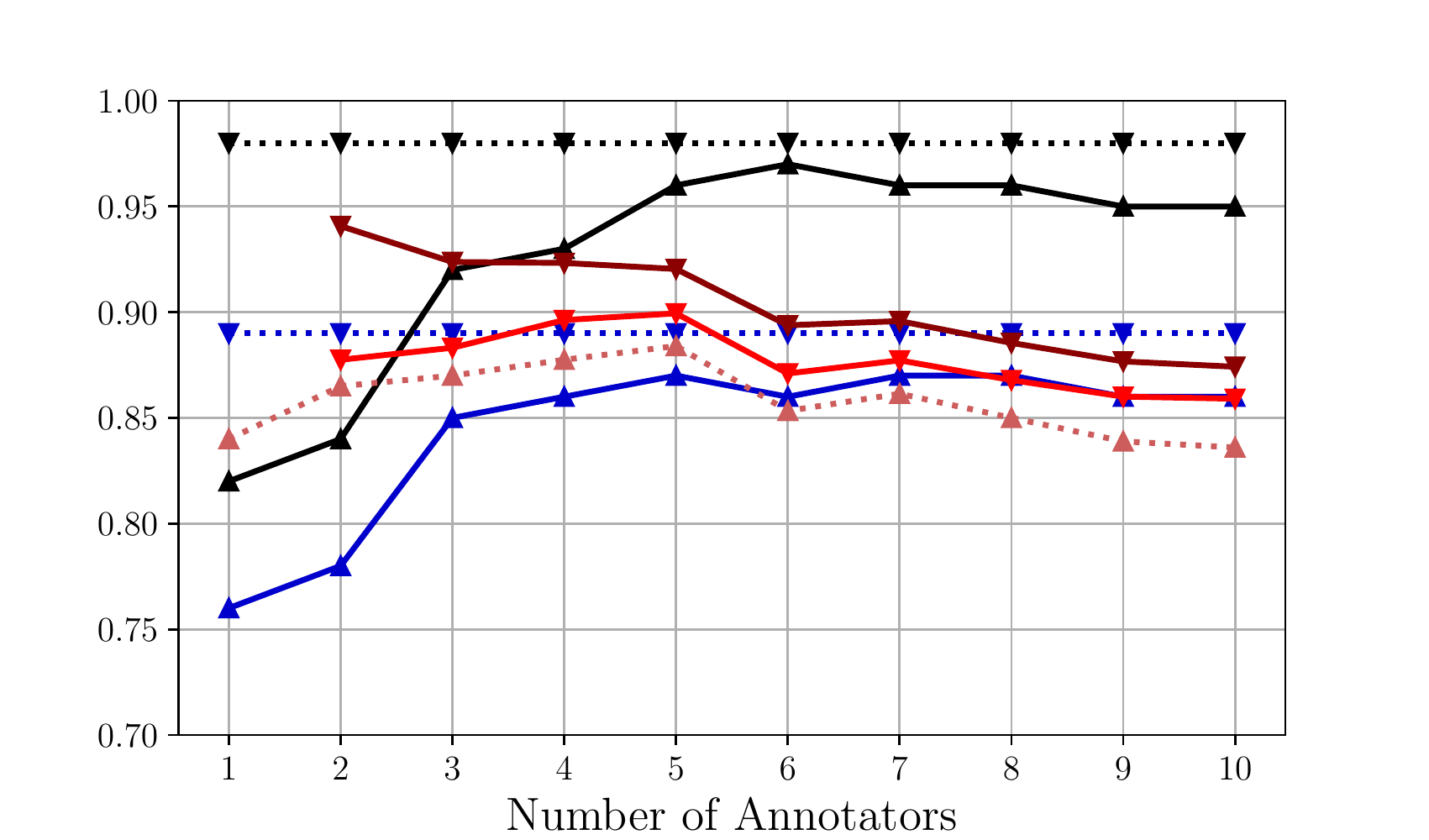}}
  \end{minipage}
  \hfill
  \begin{minipage}{.43\linewidth}
    \centering
    \subcaptionbox{Shape Classification}
      {\includegraphics[trim=10mm 0mm 10mm 10mm, width=\linewidth]{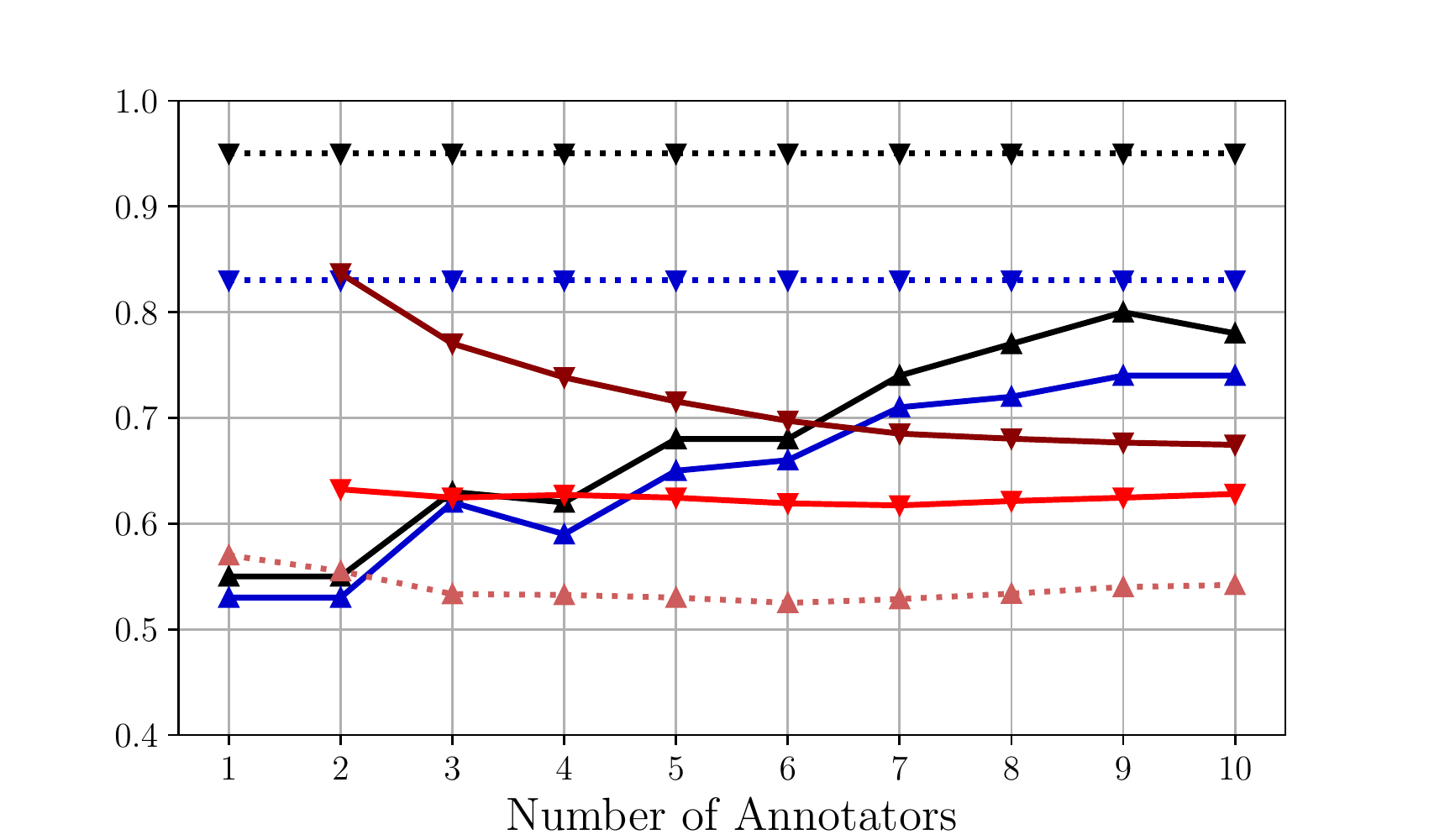}}
 \end{minipage}
 \hfill
 \begin{minipage}[t]{.12\linewidth}
    \centering
    % \frame{\includegraphics[trim=0cm 0mm 0cm 0mm, width=0.65\linewidth]{fig/legend.pdf}}
    \includegraphics[trim=0cm 22mm 0cm 0mm, width=0.95\linewidth]{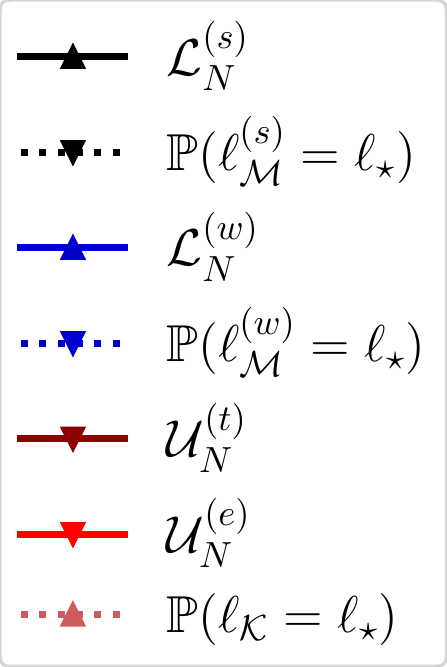}
    % \subcaptionbox{First bottom right}
    %   {\includegraphics[width=\linewidth,height=120pt]{fig/legend.pdf}}
  \end{minipage}
  \hfill
  \caption{Comparison of sample lower bound $\mathcal L_N$ for model oracle accuracy $\mathbb{P}(\ellmodel = \elloracle)$. Relatively strong and weak models are indicated by $\mathcal M^{(s)}$ and $\mathcal M^{(w)}$. Another comparison of sample theoretical upper bound $\mathcal{U}_N^{(t)}$ and sample empirical upper bound $\mathcal{U}_N^{(e)}$ of average oracle accuracy of annotators $\mathbb{P}(\ellaverage = \elloracle)$. }
   \label{fig:bounds}
\end{figure*}

%% file: tab-assumptions.tex
\begin{table}[h]
    \centering
\resizebox{\columnwidth}{!}{
    \begin{tabular}{lcc}
      \toprule
Task &  $\mathbb{P}(\ell_i = \elloracle|\ell_j = \elloracle)$ & $\mathbb{P}(\ell_i = \elloracle)$\\
\midrule
Color & \textbf{0.850} & 0.836\\
\midrule
Shape & \textbf{0.586} & 0.542\\
\bottomrule
\\[-0.5em]
\multicolumn{3}{c}{(a) Theorem~\ref{thm:basicupperbound} assumes $\mathbb{P}(\ell_i = \elloracle|\ell_j = \elloracle) \geq \mathbb{P}(\ell_i = \elloracle)$, $i \neq j$}\\ \\
\toprule
Task $b$ &  $\mathbb{P}(\ell_b = \elloracle|\ell_a \neq \elloracle)$ & $\sum_{\ellincorrect \neq \elloracle}\mathbb{P}(\ell_b = \ellincorrect|\ell_a \neq \elloracle)$\\
\midrule
Color $\mathcal M^{(w)}$ & \textbf{1.000} & 0.000\\
Color $\mathcal M^{(s)}$ & \textbf{1.000} & 0.000\\
\midrule
Shape $\mathcal M^{(w)}$ & \textbf{0.579} & 0.421\\
Shape $\mathcal M^{(s)}$ & \textbf{0.895} & 0.105\\
\bottomrule
\\[-0.5em]
\multicolumn{3}{c}{(b) Theorem~\ref{thm:basiclowerbound} assumes $\mathbb{P}(\ell_b = \elloracle|\ell_a \neq \elloracle) \geq \mathbb{P}(\ell_b = \ellincorrect|\ell_a \neq \elloracle)$}\\
 \end{tabular}
 }
% \vspace{-2mm}
    \caption{Validating our assumptions for both upper bound Theorem~\ref{thm:basicupperbound} and lower bound Theorem~\ref{thm:basiclowerbound} on Color and Shape.}
    % \vspace{-2mm}
    \label{tab:test_assumption}
\end{table}

%% file: tab-nlp-tasks.tex
\newcommand{\cspace}{~~~~~~~~~~~~~~~}
\begin{table*}[t]
\resizebox{\textwidth}{!}{
\centering
\begin{tabular}{ l cl cl cl } 
 \toprule
  & \multicolumn{2}{c}{\cspace \textsc{SST 5-Class}} & \multicolumn{2}{c}{\cspace \textsc{SST 2-Class}} & \multicolumn{2}{c}{\cspace \textsc{SNLI 3-Class}} \\
 
 B & Classifier & Score & Classifier & Score & Classifier & Score\\
 \midrule \midrule
 $\mathcal{U}_N^{(t)}$ & Avg. Human & 0.790 $^\ddagger$ & Avg. Human & 0.960 $^\ddagger$ & Avg. Human & 0.904 $^\ddagger$\\
 $\mathcal{U}_N^{(e)}$ & Avg. Human & 0.660 $^\dagger$ & Avg. Human & 0.939 $^\dagger$ & Avg. Human & 0.879 $^\dagger$\\
 \midrule \midrule
                & \shortstack{CNN-LSTM \\ \cite{zhou2015c}} & 0.492 & \shortstack{CNN-LSTM \\ \cite{zhou2015c}} & 0.878 & \shortstack{BiLSTM \\ \cite{chen2018enhancing}} & 0.855\\
 \cmidrule(lr){2-7}
   $\mathcal{L}_N$ & \shortstack{Constituency Tree-LSTM \\
                \cite{tai2015improved}} & 0.510 & \shortstack{Constituency Tree-LSTM \\ \cite{tai2015improved}}& 0.880 &  \shortstack{Tree-CNN \\ \cite{mou2016natural}} & 0.821\\
 \cmidrule(lr){2-7} 
                  & \shortstack{BERT-large \\ \cite{devlin2019bert}} & 0.555 & \shortstack{BERT-large \\ \cite{devlin2019bert}} & 0.949$^\dagger$ &  \shortstack{LM-Pretrained Transformer\\ \cite{radfordimproving}} & 0.899$^\dagger$\\
 \cmidrule(lr){2-7} 
                & \shortstack{RoBERTa+Self-Explaining \\ \cite{sun2020self}} & 0.591 & \shortstack{StructBERT \\ \cite{wang2020structbert}} & 0.971$^\ddagger$ & \shortstack{SemBERT \\ \cite{zhang2020semantics}} & 0.919$^\ddagger$\\
                % & & & & & &\\
 \bottomrule
\end{tabular}
}
\caption{The sample theoretical upper bounds and sample empirically approximated upper bounds, $\mathcal{U}_N^{(t)}$ and $\mathcal{U}_N^{(e)}$, of the average oracle accuracy of the human annotators, and the sample lower bounds $\mathcal{L}_N$ of some representative models on the SST and SNLI tasks. Those models with $\mathcal{L}_N$ higher than $\mathcal{U}_N^{(e}$ or even $\mathcal{U}_N^{(t)}$ are highlighted with $\dagger$ or $\ddagger$.} %\xqk{Shall we replace the Agent in this Table?} \cw{Good point ; I suggest Agent $\rightarrow$ Classifier so that we cover human + model}}
% \vspace{-2mm}
\label{tab:nlp_tasks}

\end{table*}

%% file: tab-confidence.tex
\begin{table}[h]
    \centering
% \resizebox{\columnwidth}{!}{
    \begin{tabular}{lccc}
      \toprule
 Model &  Task & S(HMS) & S(OMS)\\
\midrule
% \cite{devlin2019bert} & SST-2 & 0.0240& 0.0240\\
% \cite{wang2020structbert} & SST-2 & 0.7804 & 0.8670\\
% \midrule
% \cite{radfordimproving}& SNLI & 0.8630 & 0.9275\\
% \cite{zhang2020semantics} & SNLI & 0.9996 & 0.9999\\

% \cite{devlin2019bert} & SST-2 & 0.0155 & \textbf{0.0241}\\
% \cite{wang2020structbert} & SST-2 & 0.7550 & \textbf{0.8670}\\
% \midrule
% \cite{radfordimproving}& SNLI & 0.8504 & \textbf{0.9279}\\
% \cite{zhang2020semantics} & SNLI & 0.9997 & \textbf{0.9999}\\

\cite{devlin2019bert} & SST-2 & $<0$ & $<0$\\
\cite{wang2020structbert} & SST-2 & 0.4730 & \textbf{0.6208}\\
\midrule
\cite{radfordimproving}& SNLI & 0.8482 & \textbf{0.9267}\\
\cite{zhang2020semantics} & SNLI & 0.9997 & \textbf{0.9999}\\

\bottomrule
 \end{tabular}
%  }
% \vspace{-2mm}
    \caption{Confidence score $S$ for the certificated models that outperform human annotators in SST-2 and SNLI.
    % \vspace{-2mm}
    \label{tab:sota_confidence}}
\end{table}

%% file: sec5-conclusion.tex
%!TEX root = main.tex

\section{Conclusions}
In this paper, we built a theory towards estimating the oracle accuracy of classifiers. Our theory covers \textit{i)} the upper bounds for the average performance of human annotators, \textit{ii)} lower bounds for machine learning models, and \textit{iii)} confidence scores which formally capture the degree of certainty to which we may assert that a model outperforms human annotators. Our theory provides formal guarantees even within the highly practically relevant realistic setting of a finite data sample and no access to an oracle to serve as the ground truth. Our experiments on synthetic classification tasks validate the plausibility of the assumptions on which our theorems are built. Finally, our meta analysis of existing progress succeeded in identifying some existing state-of-the-art models have already achieved superhuman performance compared to the average human annotator.

%% file: sec6-impact.tex
%!TEX root = main.tex

\section*{Broader Impact}
% Positive aspect:
Our approach can identify classification models that outperform typical humans in terms of classification accuracy. Such conclusions influence the understanding of the current stage of research on classification, and therefore potentially impact the strategies and policies of human-computer collaboration and interaction. The questions we may help to answer include the following: \textit{When should we prefer a model's diagnosis over that of a medical professional? In courts of law, should we leave sentencing to an algorithm rather than a Judge?} These questions and many more like them are too important to ignore. Given recent progress in machine learning we believe the work is overdue.

% How can we ensure that a CAPTCHA can reliably distinguish human and machine?

% Negative aspect:

Yet we caution that estimating a model's oracle accuracy in this way is not \textit{free}. Our approach requires the results from multiple annotators and preferably also the number of annotators should be higher than the number of possible classes in the target classification task. Another potential challenge in applying our analysis is that some of our assumptions may not hold under some specific tasks or settings. We recommend those who apply our theory where possible to collect a small amount of `oracle' annotations, to validate the assumptions in this paper.

%% file: appendix.tex
%!TEX root = main.tex
\clearpage
\onecolumn
\appendix
\section{Proof Details}%
\label{appendix:proof}%
\paragraph{Proof of Lemma~\ref{lem:ue}}
\begin{align}
    \frac{\mathcal{U}}{\mathcal{U}^{(e)}} =& \sqrt{\frac{K-1}{K} \frac{\sum_{i=1}^{K}\sum_{j=1}^{K}\mathbb{P}(\ell_i = \ell_j)}{\sum_{i=1}^{K}\sum_{\substack{j=1\\i\neq j}}^{K}\mathbb{P}(\ell_i = \ell_j)}} \nonumber \\
    =& \sqrt{\frac{K-1}{K}} \sqrt{1+ \frac{\sum_{i=1}^{K}\mathbb{P}(\ell_i = \ell_i)}{\sum_{i=1}^{K}\sum_{\substack{j=1\\i\neq j}}^{K}\mathbb{P}(\ell_i = \ell_j)}} \nonumber \\
     =& \sqrt{\frac{K-1}{K}} \sqrt{1+ \frac{K}{\sum_{i=1}^{K}\sum_{\substack{j=1\\i\neq j}}^{K}\mathbb{P}(\ell_i = \ell_j)}} \label{eq:empirical_approx_2_terms}
\end{align}
For the first factor in \eqref{eq:empirical_approx_2_terms},
\begin{equation}
    \lim_{K\rightarrow +\infty} \sqrt{\frac{K-1}{K}} = 1. \label{eq:converge_term1}
\end{equation}
For the second factor in \eqref{eq:empirical_approx_2_terms}, as both annotators address the same task, the annotator agreement should be better than guessing uniformly at random, \textit{i.e.} $\mathbb{P}(\ell_i = \ell_j) \geq 1/N_c$, where $N_c$ is the number of classes. Then, we have
\begin{equation*}
    0 \leq \frac{K}{\sum_{i=1}^{K}\sum_{\substack{j=1\\i\neq j}}^{K}\mathbb{P}(\ell_i = \ell_j)} \leq \frac{N_c}{K-1}.
\end{equation*}
As $ \lim_{K\rightarrow +\infty} \frac{N_c}{K-1} = 0$,
\begin{equation}
    \lim_{K\rightarrow +\infty} \sqrt{1+ \frac{K}{\sum_{i=1}^{K}\sum_{\substack{j=1\\i\neq j}}^{K}\mathbb{P}(\ell_i = \ell_j)}}=1. \label{eq:converge_term2}
\end{equation}
Combining \eqref{eq:converge_term1} and \eqref{eq:converge_term2}, we have %Eq~\ref{eq:empirical_approx},
\begin{equation}
\lim_{K\rightarrow +\infty} \frac{\mathcal{U}}{\mathcal{U}^{(e)}}=1.
\end{equation}
% $\QED$

% \paragraph{Proof of Algorithm 1.}
% Given the assignment of the margins, $t_l=t_u=(\mathcal L_N-\mathcal U_N)/2$, they satisfy 
% \begin{equation}
%     \mathcal L_N - t_l - \sqrt{t_u + \mathcal U_N^2} \leq 0
% \end{equation}

% \begin{align}
%      & \mathcal L_N - t_l - \sqrt{t_u + \mathcal U_N^2} \leq 0 \nonumber \\
%     \Leftrightarrow\ \  & \mathcal L_N - t_l \leq \sqrt{t_u + \mathcal U_N^2} \nonumber \\
%     \Leftrightarrow\ \  & \mathcal \frac{(\mathcal L_N+\mathcal U_N)^2}{4} \leq \frac{\mathcal L_N-\mathcal U_N}{2} + \mathcal U_N^2 \nonumber \\
% \end{align}

\section{Details for Annotation}
\label{appendix:AMT}
We crowd source the annotations via the \textit{Amazon Mechanical Turk}. The annotation interfaces with instructions for color classification and shape classification are illustrated in Figure~\ref{fig:task_interface}.  Each example is annotated by $K=10$ different annotators. For quality control, we \textit{i)} offer our tasks only to experienced annotators with 100 or more approved HITs; \textit{ii)} automatically reject answers from annotators who have selected `None of the above'.

\begin{figure}[t]
\hfill
\begin{subfigure}[b]{1.\linewidth}
	\includegraphics[width=1.0\linewidth]{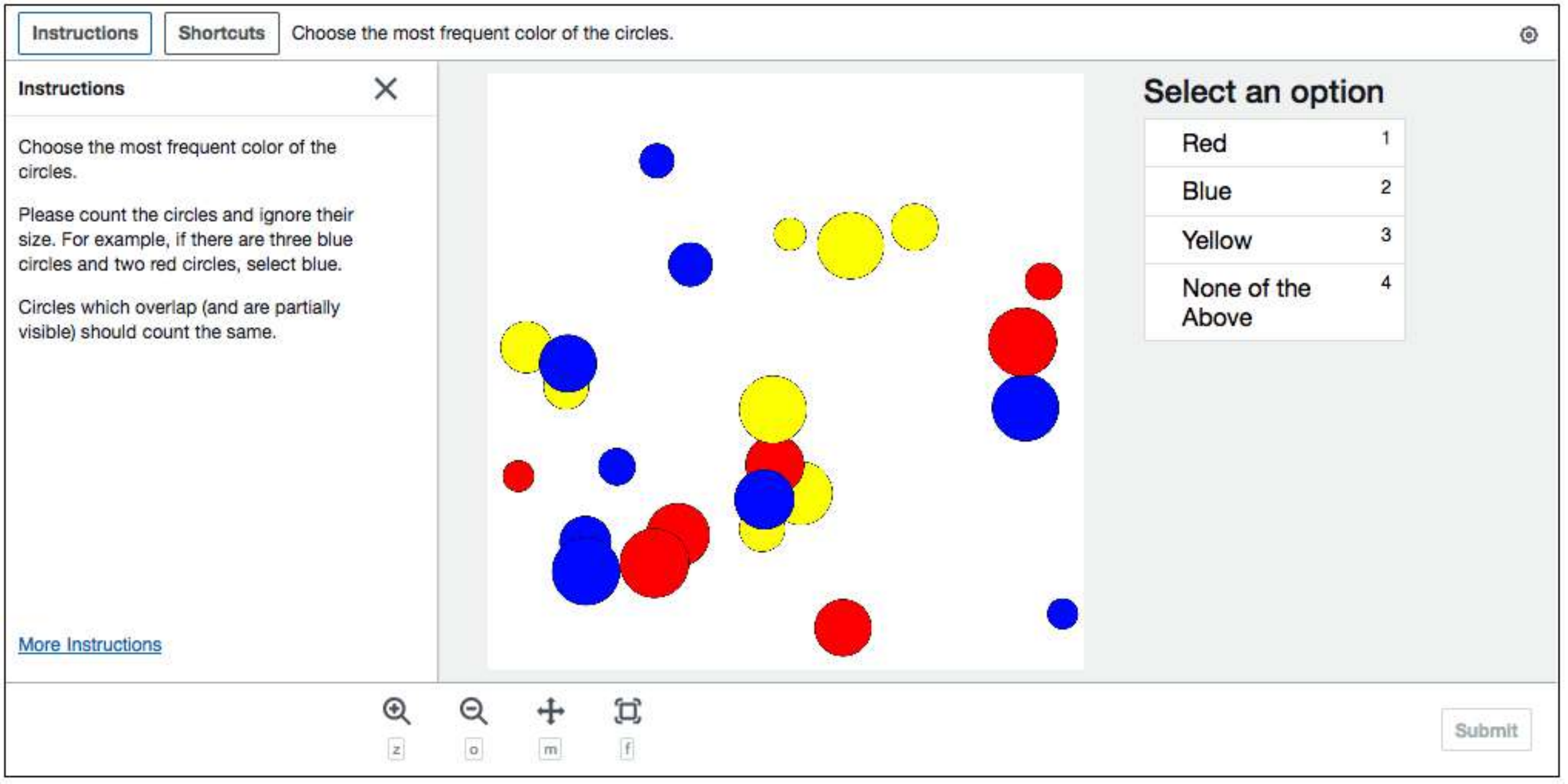}
% 	\vspace{-5mm}
	\caption{Color Classification}
	\label{fig:task_example_color}
\end{subfigure}
 \hfill
\begin{subfigure}[b]{1.\linewidth}
	\includegraphics[width=1.0\linewidth]{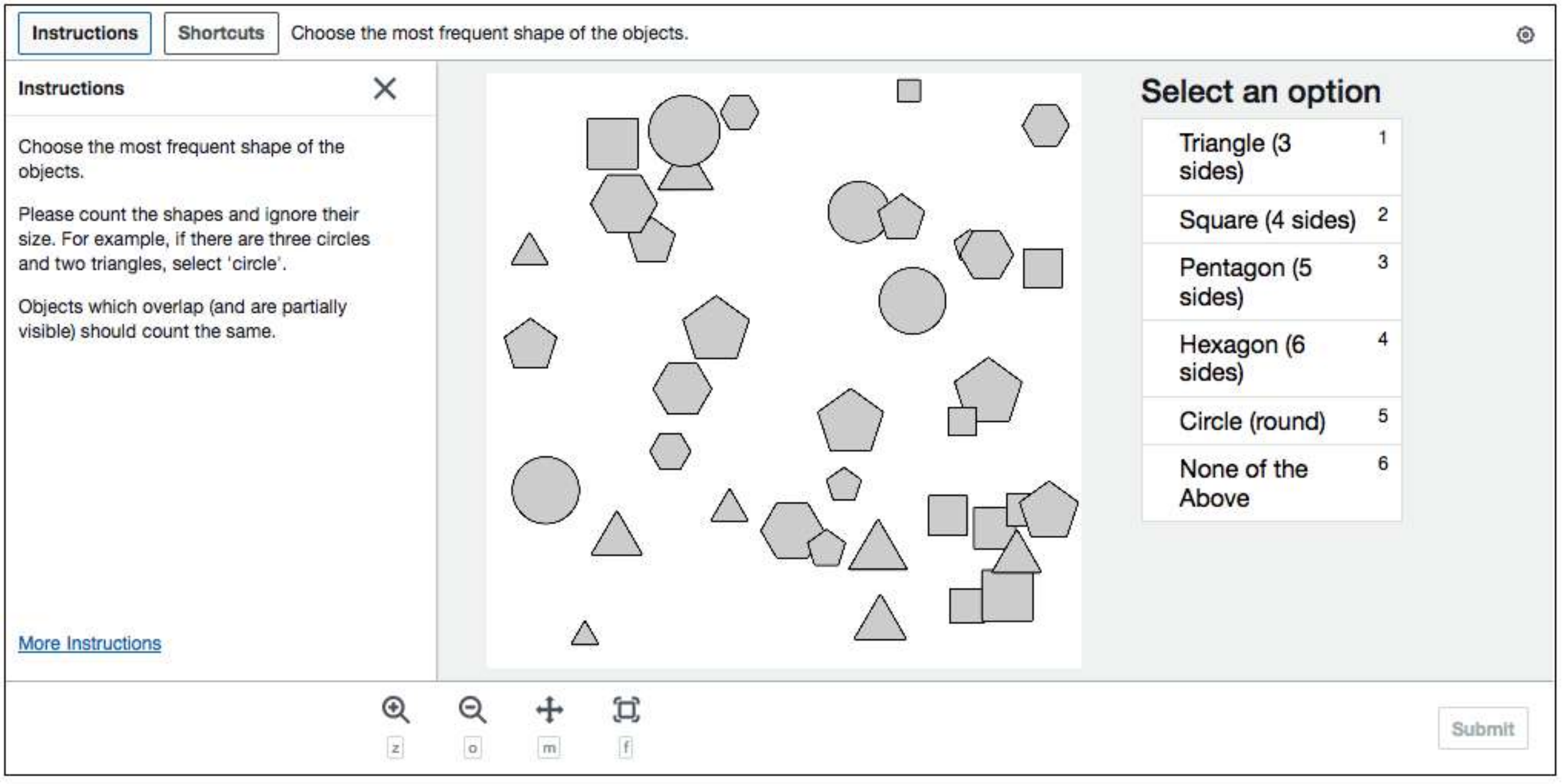}
% 	\vspace{-5mm}
	\caption{Shape Classification}
	\label{fig:task_example_shape}
\end{subfigure}
\hfill
    % \vspace{-3mm}
	\caption{Human annotator interface for the Color Classification and Shape Classification tasks.}
	\label{fig:task_interface}
    % \vspace{-3mm}
\end{figure}